\documentclass
[
    sigconf = true,     
    review = false,      
    screen = true,      
    anonymous = false,   
]
{acmart}

\allowdisplaybreaks[2]

\usepackage[utf8]{inputenc}
\usepackage{xspace}
\usepackage{amsmath,amsthm,mathtools}
\usepackage{bbm}
\usepackage{url}
\usepackage[caption=false]{subfig}
\usepackage{algorithm}
\usepackage[noend]{algpseudocode}
\usepackage{xcolor}
\usepackage{tikz}
\usepackage{graphicx}

\newcommand{\om}{\textsc{OneMax}\xspace}
\newcommand{\onemax}{\om}
\newcommand{\LO}{\textsc{Leading\-Ones}\xspace}
\newcommand{\leadingones}{\LO}
\newcommand{\LeadingOnes}{\LO}
\newcommand{\PLeadingOnes}{\textsc{PLeading\-Ones}\xspace}

\newcommand{\binval}{\textsc{BinVal}\xspace}

\newcommand{\jump}{\textsc{Jump}\xspace}
\newcommand{\Pjump}{\textsc{PJump}\xspace}
\newcommand{\PJump}{\textsc{PJump}\xspace}
\newcommand{\trap}{\textsc{Trap}\xspace}
\newcommand{\HAM}{\textsc{Ham}\xspace}

\newcommand{\oea}{\mbox{${(1 + 1)}$~EA}\xspace}

\newcommand{\opllga}{\mbox{${(1+(\lambda,\lambda))}$~GA}\xspace}
\newcommand{\ollga}{\opllga}

\DeclareMathOperator{\pow}{Pow}
\DeclareMathOperator{\Poi}{Poi}
\DeclareMathOperator{\Id}{Id}

\newcommand{\olsi}[1]{\,\overline{\!{#1}}}

\newcommand{\R}{\ensuremath{\mathbb{R}}}

\newcommand{\Z}{\ensuremath{\mathbb{Z}}}

\let\originalleft\left
\let\originalright\right
\renewcommand{\left}{\mathopen{}\mathclose\bgroup\originalleft}
\renewcommand{\right}{\aftergroup\egroup\originalright}

\copyrightyear{2022}
\acmYear{2022}
\setcopyright{licensedothergov}\acmConference[GECCO '22]{Genetic and Evolutionary Computation Conference}{July 9--13, 2022}{Boston, MA, USA}
\acmBooktitle{Genetic and Evolutionary Computation Conference (GECCO '22), July 9--13, 2022, Boston, MA, USA}
\acmPrice{15.00}
\acmDOI{10.1145/3512290.3528720}
\acmISBN{978-1-4503-9237-2/22/07}

\acmPrice{15.00}
\acmISBN{978-1-4503-XXXX-X/18/06}

\begin{document}

\title[Towards a Stronger Theory for Permutation-based Evolutionary Algorithms]{Towards a Stronger Theory for Permutation-based Evolutionary Algorithms}



\author{Benjamin Doerr}
\affiliation{%
  \institution{Laboratoire d'Informatique (LIX), CNRS, \'Ecole Polytechnique,\\ Institut Polytechnique de Paris}
  \city{Palaiseau}
  \country{France}
}  

\author{Yassine Ghannane}
\affiliation{%
  \institution{\'Ecole Polytechnique,\\ Institut Polytechnique de Paris}
  \city{Palaiseau}
  \country{France}
  }

\author{Marouane Ibn Brahim}
\affiliation{%
  \institution{\'Ecole Polytechnique,\\ Institut Polytechnique de Paris}
  \city{Palaiseau}
  \country{France}
  }

{\sloppy
\begin{abstract}
While the theoretical analysis of evolutionary algorithms (EAs) has made significant progress for pseudo-Boolean optimization problems in the last 25 years, only sporadic theoretical results exist on how EAs solve permutation-based problems.

To overcome the lack of permutation-based benchmark problems, we propose a general way to transfer the classic pseudo-Boolean benchmarks into benchmarks defined on sets of permutations. We then conduct a rigorous runtime analysis of the permutation-based $(1+1)$ EA proposed by Scharnow, Tinnefeld, and Wegener (2004) on the analogues of the \textsc{LeadingOnes} and \textsc{Jump} benchmarks. The latter shows that, different from bit-strings, it is not only the Hamming distance that determines how difficult it is to mutate a permutation $\sigma$ into another one $\tau$, but also the precise cycle structure of $\sigma \tau^{-1}$. For this reason, we also regard the more symmetric scramble mutation operator. We observe that it not only leads to simpler proofs, but also reduces the runtime on jump functions with odd jump size by a factor of $\Theta(n)$. Finally, we show that a heavy-tailed version of the scramble operator, as in the bit-string case, leads to a speed-up of order $m^{\Theta(m)}$ on jump functions with jump size~$m$.%
\end{abstract}




\maketitle
\section{Introduction}

Mathematical runtime analyses have raised our understanding of evolutionary algorithms for many years now (see~\cite{DrosteJW02} for an early, very influential work in this field). They have explained their working principles, have given advice on how to set their parameters, and have even lead to the development of new operators and algorithms. 

A closer look at these results~\cite{NeumannW10,AugerD11,Jansen13,DoerrN20}, however, reveals that the vast majority of these works only consider bit-string representations, that is, the search space is the space $\Omega = \{0,1\}^n$ of bit strings of length~$n$. Hence for the practically also relevant case of permutation-based optimization problems~(see, e.g.,~\cite{EibenS15}), that is, the search space is the set $S_n$ of permutations of $[1..n] := \{1, \dots, n\}$, our rigorous understanding is much less developed (see Section~\ref{sec:previous} for a detailed account of the state of the art). This shortage is visible, e.g., from the fact that there are no established benchmark problems except for the sorting problem and there are no mathematical results on how to set the parameters of permutation-based evolutionary algorithms. 

With this work, we aim at contributing to the foundations of a systematic and principled analysis of permutation-based evolutionary algorithms. Noting that the theory of evolutionary algorithms for bit-string representations has massively profited from the existence of widely accepted and well-understood benchmarks such as \onemax, \binval, linear functions, \leadingones, royal road functions, \trap, \jump, and many others, we first propose a simple generic way to translate benchmarks defined on bit strings into permutation-based benchmarks. 

Since the resulting permutation-based \onemax problem is equivalent to a sorting problem regarded in~\cite{ScharnowTW04}, we proceed with mathematical runtime analyses of the two next most prominent benchmarks \leadingones~\cite{Rudolph97} and \jump~\cite{DrosteJW02}. As algorithm, we consider the permutation-based \oea of~\cite{ScharnowTW04} performing as mutation a Poisson-distributed number of swaps (called exchanges in~\cite{ScharnowTW04}). 

For \leadingones, without greater problems, we prove an upper bound via fitness level arguments analogous to~\cite{Rudolph97} and a lower bound via the observation that, different from the bit-string case, it is unlikely to gain more than two fitness levels while the fitness is below~$\frac n2$. This observation saves us from counting so-called free-riders as in~\cite{DrosteJW02}. The final result is a $\Theta(n^3)$ runtime guarantee for the permutation-based \oea on this \leadingones benchmark. Given that the probability of a fitness improvement in the permutation-based case is $\Theta(n^{-2})$ (as opposed to $\Theta(n^{-1})$ in the bit-string case), this runtime estimate, higher by a factor of $\Theta(n)$ than for the bit-string case, is very natural. 

Our analysis for jump functions, in contrast, reveals a subtle difference to the bit-string case. Similar to the bit-string case, also in the optimization of a permutation-based jump function, the most difficult step is to mutate a local optimum into the global optimum, which is the only improving solution here. This requires flipping $m$ particular bits in the bit-string case and permuting $m$ particular elements in the permutation-case, where $m$ is the jump size parameter of the jump function. Different from the bit-string case, the probability that one application of the mutation operator achieves this goal depends critically on the current permutation, more precisely, on its cycle structure. Consequently, the success probability for this event can be as low as $\Theta(n^{-2(m-1)})$ and as high as $\Theta(n^{-2 \lceil m/2 \rceil})$. By analyzing the random walk on the plateau formed by the local optima, we manage to show a runtime guarantee of only $\Theta(n^{2 \lceil m/2 \rceil})$, but this analysis is definitely more involved than for the bit-string case. 

Both from the complicated analysis and the slightly odd result that jump functions with jump size $m$ and $m+1$, $m$ odd, have the same asymptotic optimization time, we were led to wonder if the mutation operator regarded in~\cite{ScharnowTW04} is really the most appropriate one. We therefore also considered a variant of the scramble mutation operator, which randomly permutes a subset of the ground set. To be comparable with the previous operator, we choose again a number $k$ from a Poisson distribution with expectation $\lambda = 1$, then choose a random set of $k$ elements from the ground set $[1..n]$, and randomly permute these in our given permutation. For this operator, we prove that the runtime of the \oea on jump functions with jump size $m$ becomes $\Theta(n^m)$ regardless of the parity of~$m$, hence a factor of $\Theta(n)$ less when $m$ is odd. Both from the more natural result and the easier proof, we would speculate that this is a superior way of performing mutation on permutation spaces. 

Finally, we analyze the performance of a heavy-tailed variant of the scramble mutation operator. For bit-string representations, it was observed in~\cite{DoerrLMN17} that heavy-tailed mutation operators, and more generally heavy-tailed parameter choices~\cite{AntipovBD21gecco}, can greatly speed up the runtime of evolutionary algorithms. In particular, on jump functions with gap size $m$ the \oea with a heavy-tailed mutation rate was shown to be by a factor of $m^{\Theta(m)}$ faster than with the standard mutation rate $\frac 1n$. We show the same result for permutation-based jump functions: Choosing the number $k$ in the scramble operator not according to a Poisson distribution with expectation $\lambda=1$, but from a power-law distribution on $[1..n]$, gives a speed-up of order $m^{\Theta(m)}$.

In summary, our results on the \leadingones and \jump benchmarks show that several arguments and methods from the bit-string world can easily be extended to permutation search spaces, however, the combinatorially richer structure of the set of permutations also leads to new challenges and new research problem such as what is the best way to perform mutation. From our results on \jump functions, we would rather suggest to use scramble mutations than swap mutations, and rather with a heavy-tailed mutation strength than with a Poisson distributed one. We hope that our general way to translate bit-string benchmarks into permutation-based benchmarks eases the future development of the mathematical analysis of permutation-based evolutionary algorithms, a subfield where, different from bit-string representations, many fundamental questions have not yet been studied under a theoretical perspective.

\section{Previous Work}\label{sec:previous}

In this section, we describe the most relevant previous works. For reasons of space, we only concentrate on runtime analysis works, knowing well that other theoretical aspects have been studied for permutation problems as well. Since the theory of evolutionary algorithms using bit-string representations has started with and greatly profited from the analysis how simple EAs optimize polynomial-time solvable problems, we mostly focus on such results.

To the best of our knowledge, the first mathematical runtime analysis for a permutation-based problem is the study of how the \oea can be used to sort an array of $n$ elements, which is formulated at the optimization problem of maximizing the sortedness of a permutation~\cite{ScharnowTW04}. In that work, several mutation operators are proposed for permutations. Imitating the classic bit-wise mutation operator with mutation rate $\frac 1n$, which flips a number of bits that asymptotically follows a Poisson law with expectation $\lambda = 1$, a random number $k$ is chosen according to such a Poisson law and then $k+1$ elementary mutations are performed\footnote{The change from the natural value $k$ to $k+1$ was done in~\cite{ScharnowTW04} because for the problems regarded there, a mutation operation that returns the parent, that is, the application of $k=0$ elementary mutations, cannot be profitable. It is easy to see, however, that all results in~\cite{ScharnowTW04} remain valid when using $k$ elementary mutations as mutation operator.}. As elementary mutations, exchanges of two neighboring elements (called ``swap'' in~\cite{ScharnowTW04}), exchanges two arbitrary elements (called ``exchange'' in~\cite{ScharnowTW04}, but ``swap'' in the textbook~\cite{EibenS15}), jumps and reversals were proposed. Since the majority of the results in~\cite{ScharnowTW04} concern exchange mutations, we shall only discuss these here. We shall adopt the language of~\cite{EibenS15} though and call these ``swaps''. A swap thus swaps two random different elements in the word notation of a permutation, or, equivalently, replaces the current permutation $\sigma$ by $\tau \circ \sigma$, where $\tau$ is a random transposition ($2$-cycle) on the ground set $[1..n]$.

We omit the results for some measures of sortedness and only state the result most relevant for our work, namely that if the sortedness is measured by the number of items that are placed correctly, that is, the fitness is $\HAM(\sigma) = |\{i \in [1..n] \mid \sigma(i) = i\}|$, then the \oea with swap-based mutation operator takes an expected number of $\Theta(n^2 \log n)$ iterations to sort a random permutation. 

The seminal work~\cite{ScharnowTW04} has seen surprisingly little follow-up work on permutation-based EAs. There is a second early work on sorting~\cite{DoerrH08} regarding a tree-based representation and a series of works on how the choice of the (problem-specific) mutation operator influences the complexity of computing Eulerian cycles~\cite{Neumann08,DoerrHN07,DoerrKS07,DoerrJ07gecco}. In~\cite{CorusDEL18}, the sorting problem appears in one of several applications of the level-based method to analyze non-elitist algorithms. In~\cite{GavenciakGL19}, sorting via swaps in the presence of noise is investigated. Finally, in~\cite{BassinB20} it is discussed how to adjust the \ollga to permutation spaces and then an $O(n^2)$ runtime of the resulting algorithm on the sorting problem with $\HAM$ fitness is proven. Slightly less related to the focus of this work, there is an interesting a sequence of results on how EAs optimize NP-hard variants of the travelling salesman problem (TSP) in the parameterized complexity paradigm~\cite{CorusLNP16, SuttonN12, SuttonNN14}, works on finding diverse sets of TSP solutions~\cite{DoBNN20,DoGNN21}, a fixed-budget analysis for the TSP~\cite{NallaperumaNS17}, and a result on how particle swarm algorithms solve the sorting problem~\cite{MuhlenthalerRSW21}.

In summary, there are a few runtime analyses for permutation search spaces, however much fewer than for bit-string representations and strongly concentrated on very few problems.

\section{Preliminaries: Basic Notation, Permutations, and the Permutation-based \oea}\label{sec:prelims}

In this section, we define the notation used in the remainder of the paper and we describe the permutation-based \oea from~\cite{ScharnowTW04}. 

We write $[a..b] := \{z \in \Z \mid a \le z \le b\}$ to denote the set of integers between $a$ and $b$, where $a$ and $b$ can be arbitrary real numbers. We denote the problems size of an algorithmic problem by $n$. When using asymptotic notations such as $O(\cdot)$ or $\Theta(\cdot)$, these will be with respect to $n$, that is, for $n$ tending to $\infty$.

A mapping $\sigma : [1..n] \to [1..n]$ is called \emph{permutation} (of $[1..n]$) if it is bijective. As common, we denote by $S_n$ the set of all permutations of $[1..n]$. Different from some branches of algebra and combinatorics that regard permutation groups, we use the standard composition $\circ$ of permutations: For $\sigma, \tau \in S_n$, the permutation $\tau \circ \sigma$ is defined by $(\tau \circ \sigma)(i) = \tau(\sigma(i))$ for all $i \in [1..n]$. 

We recall that there are two common notations for {permutations}. The most intuitive one is to describe the permutation $\sigma \in S_n$ via the vector (``word'') of its images, that is, we write $\sigma = (\sigma(1), \sigma(2), \dots, \sigma(n))$. 
To understand the structure of a permutation, the \emph{cycle notation} is more convenient. A \emph{cycle} of length $k$, also called $k$-cycle, is a permutation $\sigma \in S_n$ such that there are pair-wise distinct elements $i_1, \dots, i_k \in [1..n]$ such that $\sigma(i_j) = i_{j+1}$ for all $j \in [1..k-1]$, $\sigma(i_k) = i_1$, and $\sigma(i) = i$ for all $i \in [1..n] \setminus \{i_1, \dots, i_k\}$. The notation $\sigma = (i_1 \dots i_k)$ is standard for such a cycle. Two cycles $(i_1 \dots i_k)$ and $(j_1 \dots j_\ell)$ are called disjoint if they are moving different elements, that is, if $\{i_1, \dots, i_k\}$ and $\{j_1, \dots, j_\ell\}$ are disjoint sets. Every permutation can be written as composition of disjoint cycles of length at least $2$. This cycle notation is unique apart from the order of the cycles in the composition, which however is not important since disjoint cycles commute, that is, satisfy $\sigma \circ \tau = \tau \circ \sigma$. To ease the writing, the $\circ$ symbols are usually omitted in the cycle notation. For example $\sigma = (12)(345)$ is the cycle notation of the permutation $\sigma = (2,1,4,5,3)$ in word notation. We finally recall the fact that every permutation $\sigma \in S_n$ can be written as composition of (usually not disjoint) $2$-cycles (called transpositions). This writing is not unique. For a $k$-cycle $\sigma = (i_1 \dots i_k)$, a shortest way to write it as composition of transpositions uses $k-1$ transpositions, e.g., $\sigma = (i_1 i_2) \circ (i_2 i_3) \circ \dots \circ (i_{k-1} i_k)$. Consequently, a permutation that is the product of $\ell$ disjoint cycles of lengths $k_1, \dots, k_\ell$ can be written as product of $\sum_{i=1}^{\ell} (k_i - 1)$ transpositions, but not of fewer. 

We finally discuss the evolutionary algorithm (EA) considered in this study. As in most previous theoretical works, we shall regard a very simple EA. This is justified both by the fact that many questions cannot be answered for more complicated algorithms and by the fact that simple algorithms consisting essentially of only one component allow a more focused study of this component. With this reasoning, as in the classic first theory works on EAs for bit-string representations, we shall regard the \oea, which is essentially a hill-climber using a mutation operator to create new solutions. In this sense, we are following the approach of the first runtime analysis work on permutation-based EAs~\cite{ScharnowTW04}. As sketched in the introduction already, a number of different mutation operators was proposed in~\cite{ScharnowTW04}, but the most promising results were obtained by building on swap operations. We first note that if $\sigma = (i_1, \dots, i_n)$ in word notation and $\tau$ is the transposition swapping $i_k$ and $i_\ell$ (that is, $\tau = (i_k i_\ell)$ in cycle notation), then $\tau \circ \sigma = (j_1, \dots j_n)$ with $j_k = i_\ell$, $j_\ell = i_k$, and $j_a = i_a$ for all $a \in [1..n] \setminus \{i_k,i_\ell\}$. In other words, we obtain the word representation for $\tau \circ \sigma$ by swapping $i_k$ and $i_\ell$ in the word representation of $\sigma$. 

It is clear that a local mutation operator such as a single random swap is not enough to let an EA leave local optima. Noting that the classic bit-wise mutation operator for bit-string representations (that flips each bit independently with probability~$\frac 1n$) performs a number of local changes (bit-flips) that asymptotically follows a Poisson law with parameter $\lambda = 1$, the authors of~\cite{ScharnowTW04} argue that it is a good idea in the permutation-case to sample a number $k \sim \Poi(1)$ and then perform $k$ random swap operation. Since in their application mutation operations that return the parent cannot be useful, they exclude the result of zero swaps by deviating from this idea and instead performing $k+1$ random transposition. To ease the comparison with the bit-string case, we shall not follow this last idea and perform instead $k \sim \Poi(1)$ random transpositions as mutation operation. We note that in many EAs for bit-string representations, zero-bit flips cannot be profitable as well, but nevertheless the standard bit-wise mutation operator is used, which with constant probability flips no bit.

With these considerations, we arrive at the permutation-based \oea described in Algorithm~\ref{alg:oea}.
\begin{algorithm}
\caption{The permutation-based \oea for the maximization of a given function $f : S_n \to \R$. It is itentical to the one in~\cite{ScharnowTW04} except that we perform only $k$ random swaps, not $k+1$.}
\label{alg:oea}
\begin{algorithmic}[1] 
\State Choose $\sigma\in S_n$ uniformly at random
\Repeat
\State Choose $k \sim \Poi(1)$
\State Choose $k$ transpositions $T_1,T_2,...,T_k$ independently and uniformly at random
\State $\sigma'\leftarrow T_k \circ T_{k-1} \circ ... \circ T_1\circ \sigma$
\If{$f(\sigma')\geq f(\sigma)$}
\State $\sigma \leftarrow \sigma'$
\EndIf
\Until forever
\end{algorithmic}
\end{algorithm}

\section{Benchmarks for Permutation-based EAs}\label{sec:construction}

As discussed in the introduction, the theory of evolutionary computation has massively profited from having a small, but diverse set of benchmark problems. These problems are simple enough to admit mathematical runtime analyses for a broad range of algorithms including more sophisticated ones such as ant colony optimizers or estimation-of-distribution algorithms. At the same time, they cover many aspects found in real-world problems such as plateaus and local optima. Being synthetic examples, they often come with parameters that allow one to scale the desired property, say the radius of attraction of a local optimum.

Such an established and generally accepted set of benchmarks is clearly missing for permutation-based EAs, which might be one of the reasons why this part of EA theory is less developed. To overcome this shortage, and to do this in a natural and systematic manner, ideally profiting to the maximum from the work done already for EAs using bit-string representations, we now propose a simple way to transform benchmarks for pseudo-Boolean optimization into permutation-based problems. We are sure that future work on permutation-based EAs will detect the need for benchmarks which cannot be constructed in this way, but we are confident that our approach sets a good basis for a powerful sets of benchmarks for permutation-based EAs. 

We note that there are different classes of permutation-based problems. In problems of the \emph{assignment type}, we have two classes of $n$ elements and the task is to assign each member of the first class to a member of the second in a bijective fashion. The quadratic assignment problem or the stable marriage problem are examples for this type. In problems of the \emph{order type}, we have precedence relations that must be respected or that are profitable to be respected. Such problems occur in production planning, where a given set of jobs have to be scheduled on a given machine. Finally, in problems of the \emph{adjacency type}, it is important that certain items are placed right before another one (possibly in a cyclic fashion). The travelling salesman problem is the classic hard problem of this type, the Eulerian cycle problem is a polynomial-time solvable example. We note that the order and adjacency types were, also under these names, already described in~\cite[p.~68]{EibenS15}. Due to the different nature of these types of problems, it appears difficult to define benchmarks that are meaningful for all types. We therefore restrict ourselves to defining benchmarks that appear suitable for the assignment type. 

In an assignment type permutation-based problem, what counts is that each element of the first class is assigned to the right element of the second class. Without loss of generality, we may assume that both classes are equal $[1..n]$. Then each possible solution to this type of problem is described by a permutation $\sigma \in S_n$. Since the way we number the elements of the original sets is arbitrary, we can without loss of generality assume that the optimal solution is the identity permutation, that is, the $\sigma$ such that $\sigma(i) = i$ for all $i \in [1..n]$. With this setup, each permutation $\sigma \in S_n$ defines a bit-string $x(\sigma)$ which indicates which of the elements are already assigned correctly, namely the string $x(\sigma) \in \{0,1\}^n$ defined by $x(\sigma)_i = 1$ if and only if $\sigma(i) = i$. Now an arbitrary $f : \{0,1\}^n \to \R$ defines a permutation-based problem $g : S_n \to \R$ via $g(\sigma) := f(x(\sigma))$ for all $\sigma \in S_n$. 

This construction immediately defines permutation-based versions of the classic benchmarks such as \onemax, \leadingones, and \jump functions. We note that the sorting problem with the $\HAM$ fitness function regarded in~\cite{ScharnowTW04} is exactly what we obtain from applying this construction to the classic \onemax benchmark. We are not aware of any other classic benchmark for which the permutation-based variant (as constructed above) has been analyzed so far. Being the next most prominent benchmarks after \onemax, in the remainder of this work we shall conduct a mathematical runtime analysis for the permutation variants of the \leadingones and \jump benchmarks. 

\section{Runtime Analysis for the Permutation-\LeadingOnes Benchmark}\label{sec:LO}

We start our runtime analysis work for permutation-based EAs with an analysis of the runtime of the \oea on the permutation version of the \leadingones benchmark. 

\subsection{Definition of the Problem}

The classic \textsc{LeadingOnes} benchmark on bit-strings was defined by Rudolph~\cite{Rudolph97} as an example for a unimodal function that is harder for typical EAs than \onemax, but still unimodal. The \leadingones functions counts the number of successive ones from left to right, that is, we have \[\LeadingOnes(x)\coloneqq
\max\{i \in [0..n] \mid \forall j \in [1..i] : x_j = 1\}\]
for all $x = (x_1,...,x_n) \in \{0,1 \}^n$. 

\LeadingOnes has quickly become an intensively studied benchmark in evolutionary computation. The \oea optimizes \LeadingOnes in quadratic time, as has been shown in~\cite{Rudolph97} (upper bound) and~\cite{DrosteJW02} (lower bound).

From our general construction principle for permutation-based benchmarks proposed in Section~\ref{sec:construction}, we immediately obtain the following permutation-variant \PLeadingOnes of this problem. For all $\sigma \in S_n$, let
\begin{align*}
\PLeadingOnes(\sigma) 
&\coloneqq  \LeadingOnes(x(\sigma)) \\
 &= \max\{i \in [0..n] \mid \forall j \in [1..i] : \sigma(j)=j\}.
\end{align*}

\subsection{Runtime Analysis}

We now show that the expected runtime of the permutation-based \oea on \PLeadingOnes is $\Theta(n^3)$. As in the bit-string case, this result follows from a fitness level argument (upper bound) and the argument that a typical run will visit a linear number of fitness levels. This second argument is actually easier in the permutation setting: We can show that the probability to gain three or more levels in one iteration is so small that with constant probability this does not happen in $O(n^3)$ iterations. Hence in this time, each iteration can increase the fitness by at most two. Since any improvement takes $\Omega(n^2)$ expected time and, when assuming that no fitness gains of more than two happen, $\Omega(n)$ improvements are necessary to reach the optimum, an $\Omega(n^3)$ lower bound for the runtime follows. 

\begin{lemma}\label{LeaInc}
  In each iteration of a run of the permutation-based \oea (Algorithm~\ref{alg:oea}) on the \PLeadingOnes benchmark, the probability of a fitness improvement is at most $\frac{6}{(n-1)^2}$.
\end{lemma}

\begin{proof}
To increase the fitness via a mutation operation, it is necessary that the first element that is not in place is moved away from its position and that the correct element is moved  there. In particular, these two elements have to be among the $2k$ elements (counted with repetition) the $k$ transpositions are composed of. We recall that the probability that $k$ transpositions are applied as mutation is $\frac{1}{ek!}$. Hence the probability for this latter event is at most 
\[\sum_{k = 0}^\infty\frac{2}{e \cdot k!}\binom{2k}{2}\left(\frac{1}{n-1}\right)^2 = 6\left(\displaystyle\frac{1}{n-1}\right)^2,
\]
where we estimated, using that the second moment of a Poisson distribution is $\lambda^2 + \lambda$, that 
\begin{align*}
    \sum_{k = 0}^\infty\frac{1}{e k!}\binom{2k}{2} &=  \sum_{k = 1}^\infty\frac{k (2k-1)}{ek!}
=  2\sum_{k = 0}^\infty\frac{k^2}{ek!} - \sum_{k = 0}^\infty\frac{1}{ek!}= 3.
\qedhere
\end{align*}
\end{proof}

\begin{theorem}
  The expected runtime of the permutation-based \oea on \PLeadingOnes is $\Theta(n^3)$.
\end{theorem}

\begin{proof}
If the current state $\sigma$ is such that $\textsc{PLeadingOnes}(\sigma) = i$, then the element $i+1$ is at some position $j$ with $j > i + 1$. Thus, a transposition between $i+1$ and $j$ increases the fitness by at least 1. Picking this transposition as a random transposition has probability $ \frac{2}{n(n-1)}$. Thus the probability of increasing the fitness with one local operation (which happens with probability $\frac1e$) is at least $\frac{2}{e{n(n-1)}}$. Needing at most $n$ of such steps, the expected waiting time can be bounded from above by $e\frac{n^2(n-1)}{2} = O(n^3)$; this argument is known as Wegener's fitness level method~\cite{Wegener01}.

For the lower bound, our analysis will rely on the fact that large fitness gains occur rarely. Let us consider the event that we raise the fitness by at least 3 and call it $A_i$. Let $B_i$ be the event that elements $i+2$ or $i+3$ were in place before the mutation step. Then
\begin{align*} 
\Pr[A_i] & = \Pr[A_i \mid B_i]\Pr[B_i] + \Pr[A_i \mid \olsi{B_i}]\Pr[\olsi{B_i}] \\
&\leq \Pr[A_i \mid B_i]\Pr[B_i] + \Pr[A_i \mid \olsi{B_i}].
\end{align*}

To increase the fitness by at least 3, when neither $i+2$ nor $i+3$ were in place, we need that $i+1$, $i+2$ and $i+3$ be amongst the elements touched by some transposition of the mutation step. We can hence bound $ \Pr[A_i \mid \olsi{B_i}]$ by
\begin{align*} 
\Pr[A_i \mid \olsi{B_i}] 
& \leq \sum_{k = 0}^\infty\frac{3!}{e \cdot k!}\binom{2k}{3} \left(\frac{1}{n-1}\right)^3 \\
& \leq \frac{4}{(n-1)^3} \left(2\sum_{k = 0}^\infty\frac{k^3}{ek!} - 3\sum_{k = 0}^\infty\frac{k^2}{ek!} + \sum_{k = 0}^\infty\frac{k}{ek!}\right) \\
& \leq {20}\left(\displaystyle\frac{1}{n-1}\right)^3,
\end{align*}
where we used that the second and third moment of a Poisson distribution with parameter $\lambda$ are $\lambda^2 + \lambda$  and $\lambda^3 + 3\lambda^2 +\lambda$.

Similarly, to increase the fitness in general, we need that $i+1$ and $\sigma(i+1)$ be amongst the elements touched by a transposition. Hence, by Lemma \ref{LeaInc},
\begin{align*}  
\Pr[A_i \mid B_i]  \le & \sum_{k = 0}^\infty\frac{2}{e \cdot k!}\binom{2k}{2}\left(\frac{1}{n-1}\right)^2 = 6\left(\displaystyle\frac{1}{n-1}\right)^2.
\end{align*} 

Finally, to estimate $\Pr[B_i]$, we note that, for a permutation $\sigma$ and until reaching \textsc{PLeadingOnes}($\sigma$) = $i$, the elements $i+2, i+3, \dots, n$ play symmetric roles for the decisions taken by the algorithm. Hence $i+2$ and $i+3$ are equally likely to be at any position $i+2$ through $n$, and thus $\Pr[B_i] \leq \Pr[\sigma(i+2)=i+2] + \Pr[\sigma(i+3)=i+3] \leq 2\frac{1}{n-i-1}$.
Putting these estimates together, we obtain $\Pr[A_i] \leq \frac{44}{(n-1)^3}$ for all $i \le \frac n2- \frac 12$.

Since we aim at an asymptotic result, let us assume that $n$ is at least~$4$. Let $E$ be the event of reaching fitness greater than $\frac n2 - \frac 12$, that is, at least $\frac n2$ in at most $t = \lfloor \frac{(n-1)^3}{m} \rfloor$ steps starting from a fitness of $0$, where $m$ is a constant we will explicit later. Let $F$ be the event of having at least one fitness increase of at least $3$ during this time span. If $F$ does not occur, we  need at least $\Delta = \lceil \frac n4 \rceil$ fitness improvements, giving the following bound for $m$ sufficiently large. 
\begin{align*}  \Pr[E] & \leq \Pr[F] + \Pr\left[E|\olsi{F}\right]&&\\
& \leq  t \frac{44}{(n-1)^3} + \binom{t}{\Delta} \left(\frac{6}{(n-1)^2}\right)^\Delta\\
& \leq \frac{44}{m} + \left(\frac{\frac{(n-1)^3}{m}e}{\Delta}\right)^\Delta \left(\frac{6}{(n-1)^2}\right)^\Delta\\
& \leq \frac{44}{m} + \left(\frac{24e}{m}\right)^\frac{n}{4} \le \frac 12.
\end{align*}

Since $n \ge 4$, the initial random permutation has fitness 0 with probability at least $\frac 34$. Hence the expected time to reach a fitness of at least $\frac n2$ from a random initial permutation is at least $\frac 34 \Pr[\overline E] (t+1) = \Omega\left(n^3\right)$. Thus, also the unconditional expected runtime is $\Omega\left(n^3\right).$
\end{proof}

\section{Runtime Analysis for the Permutation-Jump Benchmark}\label{sec:jump}

We proceed with a runtime analysis of the permutation variant of the \jump benchmark. In contrast to our analysis for \leadingones, where mild adaptations of the proofs for the bit-string case were sufficient, we now observe substantially new phenomena, which require substantially more work in the analysis. In particular, different from the bit-string case, where all local optima of the jump function were equivalent, now the cycle structure of the local optimum is important. Consequently, the probability of jumping from a local optimum to the global one in one iteration can range from $\Theta(n^{-2(m-1)})$ to $\Theta(n^{-2\lceil m/2 \rceil})$, where $m$ is the (constant) jump parameter. By analyzing the random walk which the \oea performs on the set of local optima while searching for the global optimum, we shall nevertheless prove a runtime of order $\Theta(n^{2\lceil m/2 \rceil})$ only.

\subsection{Definition of the Problem}

The \jump benchmark as pseudo-Boolean optimization problem was proposed in~\cite{DrosteJW02}. It is the by far most studied multimodal benchmark in the theory of evolutionary algorithm and has led to a broad set of interesting insights, mostly on crossover and on how evolutionary algorithms cope with local optima~\cite{DrosteJW02,JansenW02,DoerrDK15ecj,DoerrLMN17,HasenohrlS18,DangFKKLOSS16,DangFKKLOSS18,WhitleyVHM18,RoweA19,Doerr20gecco,AntipovDK20,RajabiW20,RajabiW21gecco,Doerr21cgajump,BenbakiBD21}. 

We now define its permutation version, following our general construction from Section~\ref{sec:construction}. To ease the notation, let $g$ denote the function that counts the number of fixed points of a permutation, that is, the number $i \in [1..n]$ of elements that are ``in place'', that is, that satisfy $\sigma(i) = i$. By our general construction principle, this is nothing else than the permutation-variant of the \onemax benchmark. The permutation-based version of the \jump benchmark, again following our general construction, is now defined as follows.

For all $n,m\in\mathbb N$, such that $m\leq n$, let \PJump$_{n,m}$ be the map from $S_n$ to $\mathbb N$ defined by
\[\PJump_{n,m}(\sigma):=\left\{\begin{array}{l}m+g(\sigma)\;\;\;\; \text{ if }g(\sigma)\leq n-m\text{ or }g(\sigma)=n,\\
                                n-g(\sigma)\;\;\;\;\text{ otherwise.}
\end{array}\right.
\]
Since a permutation cannot have exactly $n-1$ fixed points, we see that $\PJump_{n,2}$ is equal to $g+2$, hence essentially a \onemax function. For that reason, we shall always assume $m \ge 3$.

For complexity analysis, we define the sets
$$
\begin{matrix*}[l]
    A_1 = \{\sigma\in S_n \mid g(\sigma)>n-m \text{ and } g(\sigma)\neq n\},\\
    A_2 = \{\sigma\in S_n \mid g(\sigma)\leq n-m\}, \\
    A_2^+ = \{\sigma\in S_n \mid g(\sigma)=n-m\}, \\
    A_3 = \{\Id_{[1..n]}\}.
\end{matrix*}
$$
By definition, for all $ \left(\sigma_1,\sigma_2,\sigma_2^+,\sigma_3\right)\in A_1\times A_2\times A_2^+\times A_3$, we have
$$
\PJump(\sigma_1)<\PJump(\sigma_2) \leq \PJump(\sigma_2^+)<\PJump(\sigma_3).
$$

\subsection{Runtime Analysis, Upper Bound}

To prove an upper bound on the runtime of the permutation-based \oea on jump functions, we first show the following upper bound on the expected time spent on $A_2^+$, which will be the bottleneck for the runtime of the algorithm. 

\begin{theorem}\label{localglobalup}\label{thm:jumpUP}
Let $m\geq 3$ be a constant. The permutation-based \oea started in a local optimum finds the global optimum of $\PJump_{n,m}$ in an expected number of $O(n^{2\lceil\frac m2\rceil})$ iterations.
\end{theorem}

The key to prove this result is the following observation. Since we use sequences of swap operations as mutation operation, the probability that we mutate a local optimum into the global optimum heavily depends on the smallest number $\ell$ such that the local optimum can be written as product of $\ell$ transpositions. This number can range from $\lceil \frac m2 \rceil$ to $m-1$. Hence to prove a good upper bound on the time to go from a local to the global optimum, we argue that the algorithm regularly visits local optima with this shortest possible product length and then from there has a decent chance to generate the global optimum. 

For this, we shall need the following estimate for the probability of modifying the cycle structure of a given local optimum. The proof of this result is omitted for reasons of space. As all other omitted proofs, it can be found in the appendix.

\begin{lemma}\label{Dec}
If the current permutation is a local optimum, then  the probability that one iteration of the \oea changes the number of its cycles in the cycle decomposition is at most $3(\frac{m}{n-1})^2$.
\end{lemma}

We call a permutation $\sigma \in A_2^+$ \emph{good} if it consists of as many disjoint cycles as possible. This means that, apart from the $n-m$ cycles of length one, which are not that interesting, the remaining $m$ elements are permuted via (i)~a product of $m/2$ disjoint transpositions if $m$ is even, or (ii)~a product of $(m-3)/2$ disjoint transpositions and a $3$-cycle, also disjoint from these, if $m$ is odd. We first show that any $\sigma \in A_2^+$ can be transformed into a good permutation in $A_2^+$ by applying at most $m/2$ transpositions.

\begin{lemma}\label{lem:good}
  Let $\sigma \in A_2^+$. Then there is an $\ell \le \frac m2$ and a sequence of transpositions $\tau_1, \dots, \tau_\ell$ such that $\tau_\ell \circ    \dots \circ \tau_1 \circ \sigma$ is a good permutation in $A_2^+$.
\end{lemma}

\begin{proof}
  Let $c$ denote the number of cycles of odd length larger than one in the cycle decomposition of $\sigma$. Note that two such odd-length cycles can be merged by applying a transposition that contains one element from each cycle. Hence there are $c' = \lfloor \frac c2 \rfloor$ transpositions $\tau_1, \dots, \tau_{c'}$ such that $\sigma' := \tau_{c'} \circ \dots \circ \tau_1 \circ \sigma$ contains exactly $c - 2c'$ cycles of odd length larger than one (which is one such cycle if $c$ is odd and no such cycle if $c$ is even).
  
  We note that a cycle of some length $k$ can be split into a $2$-cycle and a $(k-2)$-cycle by applying one transposition. Since $\sigma'$ is the product of at least $c'$ disjoint cycles (of length larger than one) whose lengths add up to at most $m$, we see that there are $\ell' \le \frac m2 - c'$ and transpositions $\tau'_1, \dots, \tau'_{\ell'}$ such that $\tau'_{\ell'} \circ \dots \circ \tau'_1 \sigma'$ is the product of disjoint $2$-cycles and possibly one $3$-cycle (namely when $m$ is odd). This is the good permutation proving this lemma. 
\end{proof}

We are now ready to prove Theorem~\ref{thm:jumpUP}. This proof will be divided into two steps:
 \begin{enumerate}
     \item  We show that from the current local optimum, a good permutation can be reached within the next $\frac{(n-1)^2}{m}$ iterations with at least a constant probability.

    \item We give a lower bound on the probability of reaching the global optimum from a good local optimum within again $\frac {(n-1)^2}{m^2}$ iterations.
 \end{enumerate}

\begin{proof}[Proof of Theorem \ref{localglobalup}]

 {\bf Step 1:} Since we aim at an asymptotic statement, we can always assume that $n$ is sufficiently large. Let $\sigma \in A_2^+$ be the current permutation. By Lemma~\ref{lem:good}, there are $\ell \le \frac m2$ and transpositions $\tau_1,\dots, \tau_\ell$ such that $\tau_\ell \circ \dots \circ \tau_1 \circ \sigma$ is a good permutation in $A^+_2$. 
 
 Let $E$ be the event of applying this sequence of transpositions during a timespan of $t = \frac{(n-1)^2}{m}$ iterations, using mutations which keep the intermediate states unmodified in the remaining $t-\ell$ iterations. Each of these latter mutations occurs with probability at least $p_u = 1-3\left(\frac{m}{n-1}\right)^2$ by Lemma~\ref{Dec}.

 We therefore have (using that $n \geq 2$)
 \begin{align*}
 \Pr[E] 
 &\ge \binom{t}{\ell} \left(\frac{2}{en(n-1)}\right)^{\ell} p_u^{t-\ell} \\
&\ge \left(\frac 2e\right)^{\ell} 
\left(\frac{1}{\ell m}\right)^\ell 
\left(\frac{n-1}{n} \right)^{\ell} p_u^{t} \\
&\ge \left(\frac{2}{em^2}\right)^{\frac{m}{2}} p_u^{t} \\
&\ge \left(\frac{2}{em^2}\right)^{\frac{m}{2}} \left(1-3\left(\frac{m}{n-1}\right)^2\right)^{t}.
\end{align*} 
Since $\left(1-3\left(\frac{m}{n-1}\right)^2\right)^{\frac{(n-1)^2}{m}} \rightarrow e^{-3m}$ for $n$ sufficiently large, we have
\[
\Pr[E] \ge \displaystyle \frac{1}{2}\cdot\left(\frac{2}{m^2}\right)^{\frac{m}{2}}\exp{\left(-\frac{7}{2} m\right)}:=B_m.
\]

\noindent{\bf Step 2:} The second argument is a lower bound on the probability of going from a good local optimum to the global optimum in $t' = \frac{(n-1)^2}{m^2}$ steps. For this, we first observe that a good local optimum can be written as the product of $\lceil \frac m2 \rceil$ transposition (namely the disjoint transpositions the good local optimum consists of plus possibly two more for the $3$-cycle in the case that $m$ is odd). Hence the good local optimum can be mutated into the global optimum by applying $k = \lceil \frac m2 \rceil$ suitable transpositions. The probability for this is at least
\[
\frac{1}{e\lceil\frac m 2\rceil !}\frac{1}{\left(\frac{n(n-1)}{2}\right)^{\lceil\frac m2\rceil}}.
\]
To estimate the probability that this happens within $t'$ steps, we regard the $t'$ disjoint events that this happens in one iteration and that the state is not changed in the remaining $t'-1$ iterations (it is necessary that we are in a good local optimum in the iteration which shall bring us to the global optimum).

The probability of this event (assuming $n$ sufficiently large), is
\begin{align*} 
t &\frac{1}{e \lceil\frac m 2\rceil ! \left(\frac{n(n-1)}{2}\right)^{\lceil\frac m2\rceil}} \left(1-3 \left(\frac{m}{n-1}\right)^2\right)^{\frac {(n-1)^2}{m^2}-1} \\
&\ge \frac {(n-1)^2}{m^2} \frac{1}{e^4 \lceil\frac m 2\rceil !}  \left(\frac{n(n-1)}{2}\right)^{-\lceil\frac m2\rceil}:=D_{n,m}.
\end{align*}
Combining Steps~1 and~2, we see that in each interval of $C_m (n-1)^2$ iterations ($C_m:= \frac{1}{m^2} + \frac{1}{m})$, independently of what happened before, we find the optimum with probability at least $B_m D_{n,m}$.

For each positive integer $t$, let $A_t$ be the event of not reaching the global optimum in $t$ iterations. We therefore have
\begin{align*}
    \Pr[A_t] & \le \left(1-B_m D_{n,m}\right)^{\left \lfloor \frac{t}{C_m (n-1)^2} \right \rfloor}&&\\
        &\leq \exp\left(-B_m D_{n,m}\left \lfloor \frac t {C_m (n-1)^2} \right \rfloor \right).&&
\end{align*}

Thus, for $t > \lambda\frac{C_m (n-1)^2}{B_m D_{n,m}}$ for some positive real $\lambda$, we have $\Pr[A_t] \le \exp(-\lambda)$.

We conclude that the expected time for reaching the global optimum is $ O(n^{2\lceil\frac m2\rceil})$, where we recall that we treat $m$ as a constant.
\end{proof}

The result above together with elementary fitness-level arguments showing that $O(n^2 \log n)$ time suffices to reach a local optimum (omitted for reasons of space) proves our upper bound on the expected runtime.

\begin{theorem}
  Let $m \ge 3$ be a constant. The expected runtime of the permutation-based \oea on $\Pjump_{n,m}$ is $O(n^{2\lceil\frac m2\rceil})$.
\end{theorem}





\subsection{Runtime Analysis, Lower Bound}

We now prove that our upper bound from Section 6.2 is asymptotically tight. The main argument in this lower bound proof is that applying a single transposition on a permutation $\sigma$ increases the number of cycles by at most $1$, and this only if the transposition operates on elements belonging to a common cycle of $\sigma$. We first give an upper bound on the probability that a random transposition increases the number of cycles (proof omitted for reasons of space).

\begin{lemma}\label{spl}
Given a permutation $\sigma \in S_n$ with $r > 0$ distinct cycles (possibly of length one), the probability that a random transposition consists of two elements from the same cycle is at most $\frac{(n-r)(n-r+1)}{n(n-1)}$.
\end{lemma}

We are now ready to prove the main result of this section.
\begin{theorem}\label{low}
 Let $m \ge 3$ be a constant. The expected runtime of the permutation-based \oea on $\Pjump_{n,m}$ is $\Omega(n^{2\lceil\frac m2\rceil})$.
\end{theorem}

\begin{proof}
We consider first the case that the current permutation is in $A_2$, thus with $q \in [m..n]$ elements out of place. Let us call $R$ the number of cycles of length at least~$2$ in the cycle notation of $\sigma$. Consequently, the total number of cycles  is $n - q + R$. Composing by a transposition increases the number of cycles by at most~1. Thus, in order to reach the global optimum, the sequence of transpositions in a mutation step should at least be composed of $q - R$ transpositions, each raising the number of cycles from $i$ to $i+1$ with $i \in [n-q+R..n-1]$. Thus, with $k$ transpositions applied, an upper bound on the probability of reaching the global optimum from a state of fitness $n - q$ and with $R$ cycles of size $\geq 2$ is given by Lemma \ref{spl} as 
\begin{align*}
\binom{k}{q-R}&\prod_{i=0}^{q-R-1}\frac{(n - (n-q+R+i))(n - (n-q+R+i)+1)}{n(n-1)} \\
&= \frac{k! (q-R+1)!}{(k-q+R)!}\frac{1}{(n(n-1))^{q-R}}.
\end{align*}
Since $ 1 \leq R \leq \lfloor\frac q2\rfloor $, the bound becomes at most 
\[
\frac{k! q!}{(k-q+R)!}\frac{1}{(n(n-1))^{\lceil\frac q2\rceil}}.
\]

Finally, considering the random choice of $k$, we obtain an upper bound on the probability to reach the global optimum in one step from a state of fitness $n - q$ and with $R$ cycles of size $\geq 2$ of
\begin{align*}
\sum_{k = q-R}^\infty \frac{1}{e\cdot k!}\frac{k!\cdot q!}{(k-q+R)!}\frac{1}{(n(n-1))^{\lceil\frac q2\rceil}} &= \displaystyle\frac{ m!\prod_{i=m+1}^{q}i }{(n(n-1))^{\lceil\frac q2\rceil}} &&\\ &\leq \displaystyle \frac{(m+1)!}{(n(n-1))^{\lceil\frac m2\rceil}} := p.
\end{align*}
Hence, considering the fact that the bound above holds for any point in $A_{2}$, the expected time to reach the global optimum from a permutation $\sigma $ in $A_{2}$ is at least $\frac1p$ = $ \Omega (n^{2\lceil\frac m2\rceil}).$

For a random permutation, the expected number of fixed points is $1$. Thus, for $n-m \geq 1$, we estimate with Markov's inequality that having an initial random permutation with at most one fixed point and thus belonging to $A_2$ happens with probability at least $\frac{1}{2}$. Thus, the runtime is also $ \Omega (n^{2\lceil\frac m2\rceil})$ when taking into account the random initial permutation.
\end{proof}

\subsection{Scramble Mutation Operator}

Both the complexity of the proofs above and the slightly obscure result, a runtime of $\Theta(n^{2 \lceil \frac m2 \rceil})$, raise the question whether our permutation-based \oea is optimally designed. The asymmetric behavior of the different local optima suggested to us to look for a mutation operator which treats all these solutions equally. A natural choice, known in the literature on permutation-based EAs~\cite{EibenS15}, is the \emph{scramble mutation} operator, which shuffles a random subset of the ground set $[1..n]$. More precisely, this operator samples a number $k$ according to a Poisson distribution with mean $\lambda=1$, selects a random set of $k$ positions in $[1..n]$, and applies a random permutation $\rho$ to this set (formally speaking, the mutation operator returns $\rho \circ \sigma$, when $\sigma$ was the parent permutation). We note that this operator returns the unchanged parent when $k \in \{0,1\}$. We note further that we allow $\rho$ to have fixed points. Hence the Hamming distance of $\sigma$ and $\rho \circ \sigma$ could be smaller than~$k$. We do not see a problem with his. We note that one could choose $\rho$ as a random fixed-point free permutation to ensure that the Hamming distance is exactly $k$ (at the price that generating a random fixed-point free permutation is more difficult).

For this mutation operator, we shall show a runtime of $\Theta(n^m)$ on $\Pjump_{n,m}$, faster by a factor of $\Theta(n)$ compared to the swap mutation operator when $m$ is odd. Technically much easier, here without any effort we obtain bounds that are tight apart from constant factors even when allowing that $m$ is a function of~$n$.

\begin{theorem}\label{scramb}
 Let $m \geq 3$, possibly depending on $n$. The expected runtime of the permutation-based \oea with the scramble mutation operator on \Pjump$_{n,m}$ is $\Theta((m!)^2 \binom{n}{m})$.
\end{theorem}

\begin{proof}
For the upper bound, and adopting previously introduced notations, the expected time spent in $A_1$ and $A_2 \setminus A_2^+$ can again easily be bounded by $O(n^2 \log n)$ via elementary fitness level arguments. We note that both the swap and the scramble mutation operator apply a particular transposition with probability $\Theta(n^{-2})$ and such mutation steps suffice to make progress in $A_1$ and $A_2 \setminus A_2^+$.

Once the current permutation is in $A_2^+$, a mutation step which leads to the global optimum can be one operating exactly on the $m$ displaced elements and bringing them into place. Such an event occurs with probability exactly $\frac{1}{em!} \binom{n}{m}^{-1} m!^{-1}$. Thus, the expected waiting time for such an event is at most $e (m!)^2 \binom{n}{m}$.

We omit the proof of the lower bound for reasons of space.
\end{proof}

\section{Heavy-tailed Mutation Operators}

A precise runtime analysis of the classic \oea on the bit-string \jump benchmark~\cite{DoerrLMN17} has shown (i)~that the classic mutation rate of $\frac 1n$ is far from optimal for this benchmark, (ii)~that the optimal mutation rate asymptotically is equal to $\frac mn$, and (iii)~that a heavy-tailed mutation operator gives a performance very close to the optimal mutation rate, but without the need to know the gap parameter~$m$. The main reason for the speed-ups was that these mutation operators simply more often try to flip $m$ bits, which is necessary to leave the local optimum. Since choosing $k=m$ is also the best (though not the only) way to leave a local optimum of the permutation-based \jump function to a better solution, it is natural to expect similar speed-ups also with a heavy-tailed choice of $k$, and this is the main result of this section. Without going into details, we note that also choosing $k$ from a Poisson distribution with expectation $m$ would be a good choice, however, again one that works well only for jump functions with (essentially) this gap parameter. For that reason, we find the analysis of heavy-tailed mutation operators for permutation problems more interesting and do not follow the $k \sim \Poi(m)$ idea further (we note cursory that heavy-tailed parameter choices found ample uses in the recent time and often overcame in an elegant manner the problem to set one or more parameters of an evolutionary algorithm \cite{FriedrichQW18,FriedrichGQW18,QuinzanGWF21,WuQT18,AntipovBD20gecco,AntipovBD20ppsn,AntipovD20ppsn,AntipovBD21gecco,DoerrZ21aaai,CorusOY21,DoerrR22}). Since our analyses above suggest that the scramble mutation operator is more natural than the one based on swaps, we shall only regard a heavy-tailed version of the former. 

So we proceed by defining a {heavy-tailed scramble mutation} operator. We say that an integer random variable $X$ follows a \emph{power-law distribution} with parameters $\beta$ and $u$ if
\[
\Pr[X=i]=\left\{\begin{array}{l}C_{\beta,u}i^{-\beta}\;\;\;\; \text{ if }i \in [1..u],\\
                                0\;\;\;\;\text{ otherwise,}
\end{array}\right.
\]
where $C_{\beta,u} = (\sum_{k = 1}^u k^{-\beta})^{-1}$ denotes the normalization coefficient. We write $X \sim \pow(\beta, u)$
and call $u$ the range of $X$ and $\beta$ the power-law exponent.

Now we call \emph{heavy-tailed scramble mutation} (with power-law exponent $\beta$) the mutation operator that first samples a number $k \sim \pow(\beta,n)$, then selects a random subset of $k$ elements from $[1..n]$, and finally applies a random permutation on this set. Hence this operator is identical to the previously regarded scramble operator apart from the random choice of $k$, which now follows a power-law distribution instead of a Poisson distribution. 

We show the following runtime result, which implies a speed-up compared to the classic scramble operator by a factor of $m^{\Theta(m)}$.

\begin{theorem}\label{thm:scrambHT}
 Let $m \geq 3$, possibly depending on $n$. The expected runtime of the permutation-based \oea with heavy-tailed  scramble mutation on \Pjump$_{n,m}$ is $\Theta(m^{\beta} m! \binom{n}{m})$.
\end{theorem}

We omit the proof, which is very similar to the one for the regular scramble mutation operator. Both operators have a constant probability of applying a random transposition (this solves the easy fitness-level part of the analysis), but the heavy-tailed version has a much higher (by a factor of $\Omega((m!)/m^\beta)$) probability of scrambling $m$ elements, which gives the speed-up in the leading term. We note that this argument even applies to the lower-bound proof.

\section{Conclusions}

We designed a simple and general way to transfer the classic benchmarks from pseudo-Boolean optimization into permutation-based benchmarks. Our hope and long-term goal is that the theory of permutation-based EAs can profit from these in a similar manner as the classic EA theory has profited from benchmarks for bit-string representations.

As a first small step in this direction, we conduct a mathematical runtime analysis on the permutation-based \jump function class. Both from the resulting runtime and the difficulties in the proof, we deduce that the previously commonly used mutation operator of applying a random number of transpositions possibly has some drawbacks not detected so far. We overcome these difficulties by switching to the scramble mutation operator, which both leads to better runtimes and to more natural proofs. We also observe that heavy-tailed mutation strengths, proposed a few years ago for the bit-string representation, are profitable in permutation-based EAs as well. 

Given the large number of recent interesting theoretical results on EAs using a bit-string representation, we can only hope that our definition of permutation-based benchmarks paves the way for some of these to enter the world of permutation-based EAs.

\begin{acks}
This work was supported by a public grant as part of the
Investissements d'avenir project, reference ANR-11-LABX-0056-LMH,
LabEx LMH.
\end{acks}


\bibliographystyle{ACM-Reference-Format}
\bibliography{alles_ea_master,ich_master}


\clearpage
\appendix

\section{Omitted Proofs}

\begin{proof}[Proof of Lemma~\ref{Dec}]
For the number of cycles to change, by applying $k$ transpositions, at least $2$ elements among the $2k$ elements which appear in the $k$ transpositions have to be among the $m$ deranged ones. Otherwise, applying these $k$ transpositions would either lead to a permutation of inferior fitness or the cycle structure would stay the same. Hence, an iteration modifies the number of cycles with probability at most
\begin{align*} 
& \sum_{t = 0}^\infty\frac{1}{e \cdot k!}\binom{2k}{2}\left(\frac{m}{n-1}\right)^2 \le 3\left(\displaystyle\frac{m}{n-1}\right)^2,
\end{align*}
where the sum was estimated already in the proof of Lemma~\ref{LeaInc}.
\end{proof}

\begin{proof}[Proof of Lemma~\ref{spl}]
Denoting by $n_{1},\dots,n_{r}$ the lengths of the different cycles, the exact probability for this event is $p = \frac{\sum_{i = 1}^r n_{i}(n_{i}-1)}{n(n-1)}$. Hence it suffices to show that $\sum_{i = 1}^r n_{i}(n_{i}-1) \le (n-r)(n-r+1)$. To this aim, note that $f : \R^r \to \R; (n_1, \dots, n_r) \mapsto \sum_{i = 1}^r n_{i}(n_{i}-1)$ is convex. Let $e_i$ be the $i$-th unit vector of the standard basis and $\bold{1} = \sum_{i=1}^r e_r$. Then $(n_1, \dots, n_r) = \sum_{i = 1}^r \frac{n_{i}-1}{n-r}(\bold{1} + (n-r)e_i)$. By the convexity of~$f$, we have
\begin{align*}
f(n_1, \dots, n_r) & = f\left(\sum_{i = 1}^r \frac{n_{i}-1}{n-r}(\bold{1} + (n-r)e_i)\right)\\
&\le \sum_{i = 1}^r \frac{n_{i}-1}{n-r}f(1,\dots,1,n-r+1,1,\dots,1)\\
&\le (n-r)(n-r+1) \sum_{i = 1}^r \frac{n_{i}-1}{n-r}\\
&= (n-r)(n-r+1).\qedhere
\end{align*}
\end{proof}

\begin{proof}[Proof of the lower bound in Theorem~\ref{scramb}]
For the lower bound, let us suppose again that we start in $A_2$, as this happens with constant probability. Reaching the global optimum from any point of $A_2$ with $q \in [m..n]$ elements out of place demands a mutation step operating on a set containing at least these $q$ elements. Thus, a natural upper bound for the probability of this event is 
\begin{align*}
\sum_{k = q}^n & \frac{1}{e k!} \frac{\binom{n-q}{k-q}}{\binom{n}{k}} \frac{1}{k!}\\
 &= \frac{(n-q)!}{n!}\sum_{k = q}^n \frac{1}{e k!(k-q)!}\\
 & \leq \frac{(n-q)!}{n!}\frac{1}{q!} \sum_{k = q}^n \frac{1}{e(k-q)!}\\
 & \leq  \frac{(n-m)!}{n!} \frac{1}{m!}
  =  \frac{1}{m!\binom{n}{m}} \frac{1}{m!}
\end{align*}
Since this bound holds for any permutation $\sigma$ in $A_2$, the expected time until reaching the global optimum from a permutation $\sigma $ in $A_{2}$ is $\Omega ((m!)^2 \binom{n}{m})$. Then, using the argument that the initial permutation is in $A_2$ with constant probability, the unconditional expected runtime is also $\Omega ((m!)^2 \binom{n}{m})$.
\end{proof}
} 
\end{document}